\pdfoutput=1

\documentclass[letterpaper, 10 pt, conference]{ieeeconf}  

\IEEEoverridecommandlockouts                              

\usepackage{amsmath} 
\usepackage{amssymb}  
\usepackage{overpic}
\usepackage[utf8]{inputenc}
\usepackage[table]{xcolor}
\usepackage{multirow}
\usepackage{array}
\usepackage{mathtools}

\newcommand{\ls}{\hspace{0em}}     

\newtheorem{lemma}{Lemma}
\newtheorem{theorem}{Theorem}
\newtheorem{definition}{Definition}

\DeclareMathOperator*{\argmin}{arg\ min.}

\title{\LARGE \bf
Identification of Fully Physical Consistent Inertial Parameters using Optimization on Manifolds
}

\author{Silvio Traversaro$^{1}$, Stanislas Brossette$^{2}$, Adrien Escande$^{3}$, and Francesco Nori$^{1}$
\thanks{*This paper was supported by the FP7 EU projects
CoDyCo (No. 600716 ICT-2011.2.1 - Cognitive Systems
and Robotics), Koroibot (No. 611909 ICT-2013.2.1 - Cognitive Systems and Robotics) and the H2020 European project COMANOID.}
\thanks{$^{1}$Silvio Traversaro and Francesco Nori are with iCub Facility, Istituto Italiano di Tecnologia, Via Morego 30, Genoa 16163, Italy
        {\tt\small name.surname@iit.it}}%
\thanks{$^{2}$Stanislas Brossette is with CNRS-UM2 LIRMM Interactive Digital Humans, UMR5600, Montpellier, France and CNRS-AIST Joint Robotics Laboratory (JRL), UMI3218/RL, Tsukuba, Japan.
        {\tt\small name.surname@gmail.com}}%
\thanks{$^{3}$Adrien Escande is with CNRS-AIST Joint Robotics Laboratory (JRL), UMI3218/RL, Tsukuba, Japan.
        {\tt\small name.surname@gmail.com}}%
}

\newcommand{\SO}{\textrm{SO}}

\newcommand{\R}{\mathbb{R}}

\newcommand{\ipspace}{\mathfrak{P}}

\newcommand{\rmv}{{\rm v}}      
\newcommand{\rmf}{{\rm f}}      
\newcommand{\rma}{{\rm a}}      

\begin{document}

\maketitle
\thispagestyle{empty}
\pagestyle{empty}

\begin{abstract}
This paper presents a new condition, the \emph{fully physical consistency} for a set of inertial parameters to determine if they can be generated by a physical rigid body. The proposed condition ensure both the positive definiteness and the triangular inequality of 3D inertia matrices as opposed to existing techniques in which the triangular inequality constraint is ignored. This paper presents also a new parametrization that naturally ensures that the inertial parameters are \emph{fully physical consistency}. The proposed parametrization is exploited to reformulate the inertial identification problem as a \emph{manifold optimization problem}, that ensures that the identified parameters can always be generated by a physical body.
The proposed optimization problem has been validated with a set of experiments on the iCub humanoid robot. 
\end{abstract}

\section{INTRODUCTION}
A large part of existing robotic systems are modeled as a system of multiple rigid bodies. 
The knowledge of the dynamical characteristics of these rigid bodies is a key assumption of model-based control and estimation techniques. The dynamics of a rigid body, i.e. how the acceleration of a rigid body is related to the forces applied on it, is completely described by the mass distribution of the body in the 3D space. 
The mass distribution itself is completely described by 10 \emph{inertial parameters} \cite{hollerbach2008model}. These parameters may be available if a good Computer-Aided Design (CAD) model of the robot is available, but often such models are either not available, or the mass distribution of the rigid bodies in the robot changes during operation, as in the case of an end effector that grabs and heavy object. 


 Inverse robot dynamics models can be written linearly with respect to the inertial parameters of the rigid bodies composing the robot. Classical identification techniques \cite{hollerbach2008model,ayusawa2014identifiability} consider the parameters of each body to be an element of the Euclidean space $\mathbb{R}^{10}$. Exploiting this fact, the inertial parameters identification problem has been classically posed as a \emph{Linear Least Square} optimization problem \cite{hollerbach2008model}. The resulting problem is convenient from a computational point of view, but it neglects the fact that not all vectors in $\mathbb{R}^{10}$ can be generated by a physical rigid body, i.e. it is possible that some inertial parameters are identified even if no physical rigid body could generate them. 

A necessary condition for the inertial parameters to be generated by a physical rigid body was first proposed in \cite{yoshida1994}: the \emph{physical consistency} condition.  
This condition is important for control purposes because it ensures, if it is valid for all the links of a robot, the positive definiteness and the invertibility of the joint mass matrix \cite{yoshida2000}. This property is a key assumption in proving the stability of model-based control laws. 
The \emph{physical consistency} has been enforced in identification of inertial parameters  using several techniques: \cite{yoshida2000,mata2005,gautier2013identification,gautier2013positive,sousa2014physical,jovic2015identification}. However this condition is not \emph{sufficient}: it is possible that some inertial parameters that respect this condition do not correspond to any physical body: in particular this condition does not encode the \emph{triangle inequalities} of the 3D inertia matrix \cite[Chapter 3]{wittenburg2007dynamics}, as it will be explained in the remainder of the paper.

The main contribution of this paper is a new necessary and \emph{sufficient} condition for the inertial parameters to be generated by rigid body: the \emph{full physical consistency} condition. We show that this condition implies the already proposed \emph{physical consistency} condition and that the triangle inequalities are respected. Furthermore, we propose a nonlinear optimization formulation that takes into consideration this constraint by using state of the art optimization techniques on non-Euclidean manifolds \cite{brossette2015humanoid}. The proposed optimization technique is validated with a rigid body inertial identification experiment on the arm of the iCub humanoid robot.

For the sake of simplicity, in this paper, we only consider the problem of identifying the inertial parameters of a single rigid body. However, the \emph{full physical consistency} condition and the optimization on manifolds are general contributions, that could be applied to the case of the identification of inertial parameters in generic multibody structures. 

The paper is organized as follows. Section II presents the
notations  used  in  the  paper and the background on rigid body dynamics.
Section III details the proposed \emph{full physical consistency} condition, the proposed nonlinear parametrization of the inertial parameters that ensures that this condition is always satisfied and the optimization technique on the manifold of the proposed parametrization.
Section IV describe the experiments used for validation.
Remarks  and perspectives conclude the paper.

\section{BACKGROUND}
\subsection{Notation}
The following notation is used throughout the paper.
\begin{itemize}
 \item The set of real numbers is denoted by $\mathbb{R}$, while the set of nonnegative real numbers is denoted by $\mathbb{R}_{\geq 0}$. 
 \item Let $u$ and $v$ be two $n$-dimensional column vectors of real numbers, i.e. $u,v \in \mathbb{R}^n$, 
 their inner product is denoted as $u^\top v$, with ``$\top$'' the transpose operator.
\item
$\SO(3)$ denotes
the set of $\mathbb{R}^{3 \times 3}$ orthogonal matrices with determinant equal to one, namely
\begin{align*}
\SO(3) :=  \{\, R \in \mathbb{R}^{3 \times 3} \mid R^T R = I_3 , \hspace{0.3em} \operatorname{det}(R) = 1 \,\}.
\end{align*}
\item Given $u,v \in \mathbb{R}^3$, $S(u) \in \R^{3\times3}$ denotes the skew-symmetric matrix-valued operator associated with the cross product in 
  $\mathbb{R}^3$, such that $S(u) v = u \times v$
 \item The Euclidean norm of a vector of real numbers is denoted by $\left\|\cdot \right\|$.
\item $1_n \in \mathbb{R}^{n \times n}$ denotes the identity matrix of dimension~$n$; 
$0_n \in \mathbb{R}^n$ denotes the zero column vector of dimension~$n$; $0_{n \times m} \in \mathbb{R}^{n \times m}$ denotes the zero matrix of dimension~$n \times m$.
\item $A$ denotes an inertial frame and $B$ a body-fixed frame.
\item $p_B \in \mathbb{R}^3$ denotes the origin of the $B$ frame expressed in the inertial frame, while $\ls^A R_B \in SO(3)$ is the rotation matrix that transforms a 3D vector expressed with the orientation of the $B$ frame in a 3D vector expressed in the $A$ frame. $\omega \in \mathbb{R}^3$ denotes the body angular velocity of body $B$, defined as $S(\omega) = \ls^A R_B^\top \ls^A \dot{R}_B$.
\item $\rmv \in \mathbb{R}^6$ indicates the body twist \cite[Chapter 3]{murray1994mathematical}, i.e.
$$
\rmv = \begin{bmatrix} \ls^A R_B^\top \dot{p}_B \\ \omega \end{bmatrix}$$
\item $\rmf \in \mathbb{R}^6$ indicates an external wrench exerted on the body \cite[Chapter 3]{murray1994mathematical}, i.e. 
$$\rmf = \begin{bmatrix} f \\ \mu \end{bmatrix}$$
where $f$ is the 3D external force and $\mu$ the 3D external moment, that are expressed in frame $B$. 
\item Given $\rmv = \begin{bmatrix} v^\top & \omega^\top \end{bmatrix}^\top \in \mathbb{R}^6$,  $\rmv \bar \times^*$ is the 6D force cross product operator \cite{featherstone2008}, defined as:
$$
\rmv \bar \times^* 
= 
\begin{bmatrix}
S(\omega) & 0_{3\times3} \\
S(v)      & S(\omega) 
\end{bmatrix}
$$
\item Given a symmetric matrix $I \in \mathbb{R}^{3\times3}$ the $\operatorname{vech}$ operator denotes the serialization operation on symmetric matrices:
$$
\operatorname{vech} \left(
\begin{bsmallmatrix}
I_{xx} & I_{xy} & I_{xz} \\
I_{xy} & I_{yy} & I_{yz} \\
I_{xz} & I_{yz} & I_{zz}
\end{bsmallmatrix} \right)
=
\begin{bsmallmatrix}
I_{xx} \\
I_{xy} \\
I_{xz} \\
I_{yy} \\
I_{yz} \\
I_{zz}
\end{bsmallmatrix}
$$
\item Given a symmetric matrix $I \in \mathbb{R}^{n\times n}$, $I \succeq 0$ denotes that the matrix is positive semidefinite, i.e. that all its eigenvalues are nonnegative.  
\item $ g \in \mathbb{R}^3$ is the constant vector of gravity acceleration expressed with the orientation of the inertial frame $A$.
\item $\rma^g \in \mathbb{R}^6$ denotes the \emph{proper} body acceleration, i.e. the difference between the body acceleration and the gravity acceleration:
$$
\rma^g = \dot{\rmv} - \begin{bmatrix} ~^A R_B^\top g \\ 0_{3} \end{bmatrix}
$$
\end{itemize}

\subsection{Rigid Body Dynamics} 




The well known Newton-Euler equations for a rigid body \cite[Chapter 3]{murray1994mathematical} are given by: 
\begin{equation}
\label{eq:newtonEuler}
M \rma^g + \rmv \bar{\times}^* M \rmv = \rmf
\end{equation}

where $M \in \mathbb{R}^{6 \times 6}$ is the 6D inertia matrix (also known as \emph{spatial inertia}):
\begin{equation}
    M = \begin{bmatrix} m 1_{3} & -mS(c) \\ 
                        m S(c)           & {I}_B 
        \end{bmatrix}.
\end{equation}
Where:
\begin{itemize}
    \item $m \in \mathbb{R}$ is the mass of the rigid body,
    \item $c \in \mathbb{R}^3$ is the center of mass of the rigid body, expressed in the frame $B$,
    \item $I_B \in \mathbb{R}^{3\times3}$ is the 3D inertia matrix of the rigid body, expressed with the orientation of frame $B$ and with respect to the frame $B$ origin. 
\end{itemize} 

The 6D inertia matrix is parametrized by 10 parameters, usually called the \emph{inertial parameters} of the rigid body \cite{hollerbach2008model}, that are defined as $\pi \in \mathbb{R}^{10}$:
\begin{equation}
    \pi = 
    \begin{bmatrix}
    m \\
    m c \\
    \operatorname{vech}(I_B)
    \end{bmatrix} .
\end{equation}

The Newton-Euler equations \eqref{eq:newtonEuler} can be written linearly \cite{garofalo2013closed,hollerbach2008model} in the inertial parameters :
\begin{equation}
\label{eq:inertialRegressor}
Y(\rma^g, \rmv) \pi = M \rma^g + \rmv \bar{\times}^* M \rmv = \rmf .
\end{equation}

\subsection{Relationship between the inertial parameters and the density function}
The mass distribution of a rigid body in space is described by its density function:
\begin{equation}
\rho(\cdot): \mathbb{R}^3 \mapsto \mathbb{R}_{\ge 0} .
\end{equation}
The domain of this function is the points of body expressed in the body-fixed frame $B$.
We consider the density equal to zero for the points outside the volume of the rigid body, so we can define the domain of $\rho(\cdot)$ to be all the points in the 3D space $\mathbb{R}^3$.

The inertial parameters are obviously a functional of the density $\rho(\cdot)$, in particular, we can define the functional $\pi_{d}(\cdot) : (\mathbb{R}^3 \mapsto \mathbb{R}_{\ge 0}) \mapsto \mathbb{R}^{10}$ that maps the density function to the corresponding inertial parameters: 
\begin{IEEEeqnarray}{rCl}
\label{eq:pid}
   \pi_{d}(\rho(\cdot)) &=&
    \begin{bmatrix}
    m(\rho(\cdot)) \\
    mc(\rho(\cdot)) \\
    \operatorname{vech}(I_B(\rho(\cdot))) 
    \end{bmatrix} = \nonumber \\ &=& 
   \begin{bmatrix}
     \iiint\limits_{\mathbb{R}^3} \rho(r) dr \\
     \iiint\limits_{\mathbb{R}^3} r \rho(r) dr \\
     \operatorname{vech}\left(
     \iiint\limits_{\mathbb{R}^3} S^\top({r}) S({r}) \rho({r}) d{r} \right)
           \end{bmatrix} .
\end{IEEEeqnarray}







\subsection{3D Inertia at the Center of Mass and Principal Axes}
\label{subsec:principalAxis}

The 3D inertia at the center of mass is defined as

\begin{equation}
\label{eq:comInertiaDef}
I_C = \iiint\limits_{\mathbb{R}^3} S^\top(r-c) S(r-c) \rho(r) dr .
\end{equation}

Exploiting the fact that $S(\cdot)$ is linear, the inertia matrix with respect to the center of mass can be written as: 

\begin{equation}
I_C =  I_B +  m S(c) S(c) .
\end{equation}

This result is known as \emph{parallel axis theorem}. 

As $I_C$ is symmetric, it can be diagonalized with an orthogonal matrix $Q \in SO(3)$: 
\begin{equation}
I_C = Q \operatorname{diag}{(J)}  Q^\top .
\end{equation}

Using \eqref{eq:comInertiaDef} the diagonal matrix $\operatorname{diag}{(J)}$ (with $J \in \mathbb{R}^3$) can be written as:
\begin{IEEEeqnarray}{rCl}
\operatorname{diag}{(J)} &=& \iiint\limits_{\mathbb{R}^3} Q^\top S^\top(r-c) S(r-c) Q \rho(r) dr = \nonumber \\
     &\iiint\limits_{\mathbb{R}^3}& S^\top(Q^\top (r-c)) S(Q^\top(r-c)) \rho(r) dr .
\end{IEEEeqnarray}

The operation mapping the body point $r$ to $Q^T(r-c)$ can be interpreted as a change of 
reference frame, from the body frame $B$ to a frame $C$ (a \emph{principal axes} frame) whose origin is the center 
of mass of the body and whose orientation is one in which the $I_C$ matrix is diagonal.

By expressing with $\tilde{r} = Q^\top \left( r-c \right)$ the generic point of the body expressed in the $C$ frame, and with $\tilde{\rho}(\tilde{r})$ the density with respect to the $C$ frame, we can write 
$\operatorname{diag}{J}$ as:
\begin{IEEEeqnarray}{rCl}
\operatorname{diag}{J} &=&   \iiint\limits_{\mathbb{R}^3}
S^\top(\tilde{r}) S(\tilde{r}) \rho(\tilde{r}) d\tilde{r} .
\end{IEEEeqnarray}
The elements of the J vector are: 
\begin{IEEEeqnarray}{rCl}
J_{x} = \iiint\limits_{\mathbb{R}^3} ({\tilde{y}}^2 + {\tilde{z}}^2) \tilde{\rho}(\tilde{r}) dr' , \IEEEyessubnumber \\
J_{y} = \iiint\limits_{\mathbb{R}^3} (\tilde{x}^2 + \tilde{z}^2) \tilde{\rho}(\tilde{r}) d\tilde{r} , \IEEEyessubnumber \\
J_{z} = \iiint\limits_{\mathbb{R}^3} (\tilde{x}^2 + \tilde{y}^2) \tilde{\rho}(\tilde{r}) d\tilde{r} \IEEEyessubnumber .
\end{IEEEeqnarray}
We can write them as:
\begin{IEEEeqnarray}{rClrClrCl}
J_{x} &=& L_{y} + L_{z} , \quad
J_{y} &=& L_{x} + L_{z} , \quad
J_{z} &=& L_{x} + L_{y} .
\end{IEEEeqnarray}
\begin{IEEEeqnarray}{rClrCl}
L_{x} &=& \iiint\limits_{\mathbb{R}^3} {\tilde{x}}^2  \tilde{\rho}(\tilde{r}) d\tilde{r}, \quad
L_{y} &=& \iiint\limits_{\mathbb{R}^3} {\tilde{y}}^2  \tilde{\rho}(\tilde{r}) d\tilde{r}, 
\end{IEEEeqnarray}
\begin{IEEEeqnarray}{rCl}
L_{z} &=& \iiint\limits_{\mathbb{R}^3} {\tilde{z}}^2  \tilde{\rho}(\tilde{r}) d\tilde{r}.
\end{IEEEeqnarray}
Where $L_x , L_y , L_z$ are  the \emph{central second moments of mass} of the density $\tilde{\rho}(\tilde{r})$.
It is clear that the non-negativity of $\tilde{\rho}(\tilde{r})$ constraints $L = \begin{bmatrix} L_{x} & L_{y} & L_{z} \end{bmatrix}^\top$ 
to be non-negative as well.
Furthermore, it is possible to see that the non-negativity of $L$ induces the \emph{triangular inequalities} on $\operatorname{diag}{\left(J\right)}$:
\begin{IEEEeqnarray}{lCrlCrlCr}
\label{eq:triangularInequalities}
 J_{x} &\leq& J_{y} + J_{z} , \quad 
 J_{y} \ &\leq& \ J_{x} + J_{z}, \quad 
 J_{z} \ &\leq& \ J_{x} + J_{y} .
\end{IEEEeqnarray}

\subsection{Inertial Parameters Identification}
Assuming that $N$ values for $F$,$A_g$ and $V$ are measured, equation \eqref{eq:inertialRegressor} can be used to estimate $\pi$ solving the following optimization problem:
\begin{equation}
\label{eq:optimizationProblemLinear}
 \hat{\pi} =  \argmin_{\pi \in \mathbb{R}^{10}}\ ~ \sum_{i = 1}^N \left\| Y(\rma^g_i, \rmv_i) \pi - \rmf_i \right\|^2 . \\
\end{equation}

However, this optimization does not take into account the physical properties of the inertial parameters $\pi$. For this reason, the following definition was introduced.
\begin{definition}
\label{def:physicalConsistency}
A vector of inertial parameters $\pi$ is called \emph{physical consistent} \cite{yoshida1994,yoshida2000} if: 
\begin{IEEEeqnarray}{rClrCl}
m(\pi) &\geq& 0 , \qquad I_C(\pi) &\succeq& 0 .
\end{IEEEeqnarray}
\end{definition}
This condition has nice properties (it ensures that the matrix $M$ is always invertible), but is still possible to find some \emph{physical consistent} inertial parameters that can't be generated by a physical density. 

\section{CONTRIBUTION}
\subsection{Full physical consistency}
In this subsection, we propose a new condition for assessing if a vector of inertial parameters can be generated from a physical rigid body.  
We will show that all the constraints that emerge due to this \emph{full physical consistency} condition are due to the non-negativity on the density function.
\begin{definition}
\label{eq:fullDefinition}
A vector of inertial parameters $\pi^* \in \mathbb{R}^{10}$ is called \emph{fully physical consistent}  if: 
\begin{IEEEeqnarray}{rCl}
\exists\ \rho(\cdot) : \mathbb{R}^3 \mapsto \mathbb{R}_{\geq 0} ~ \text{s.t.} ~ \pi^* = \pi_d(\rho(\cdot)).
\end{IEEEeqnarray}
\end{definition}
This definition extends the concept of \emph{physical consistent} inertial parameters to include also all possible constraints of inertial parameters, such as the triangular inequalities \eqref{eq:triangularInequalities} of the diagonal elements of the inertia matrix.

\begin{lemma}
\label{lemma:lemma1}
If a vector of inertial parameters $\pi \in \mathbb{R}^{10}$ is \emph{fully physical consistent} if follows that it is \emph{physical consistent}, according to Definition \ref{def:physicalConsistency}. 
\end{lemma}
\begin{proof}
If $\pi$ is \emph{fully physical consistent}, then it follows that there exists $\rho(\cdot)$ such that the corresponding 3D inertia at the center of mass $I_C$ can be written as a function of $\rho(\cdot)$. The positive semi-definiteness of $m$ and $I_C$ then follows from the classical properties of mass and the inertia matrix of a rigid body, see for example subsection 3.3.3 of \cite{wittenburg2007dynamics}. 
\end{proof}

\begin{lemma}
If a vector of inertial parameters $\pi \in \mathbb{R}^{10}$ is \emph{fully physical consistent}, the associated inertia matrices at the body origin $I_B(\pi)$ and at the center of mass $I_C(\pi)$ respect the triangular inequalities \eqref{eq:triangularInequalities}. 
\end{lemma}

\begin{proof}
This lemma can be proved by writing $I_B$ or $I_C$ as a functional of the density function $\rho(\cdot)$, as in the proof of Lemma \ref{lemma:lemma1}. Once $I_B$ or $I_C$ are written as a functional of $\rho(\cdot)$, the demonstration that they respect the triangle inequality can be found in any rigid body mechanics textbook,  see for example subsection 3.3.4 of \cite{wittenburg2007dynamics}. 
\end{proof}
To get a hint of the demonstration of Lemma 2, consider that the diagonal elements of the 3D inertia matrix with respect to an arbitrary frame can still be written as the sum of two non-negative  \emph{second moments of mass}. The triangle inequality then arises in a way similar to the case of the inertia expressed in the principal axes.

\subsection{Full physical consistent parametrization of inertia parameters}
In this subsection, we introduce a novel nonlinear parametrization of inertial parameters that ensures the \emph{full physical consistency} condition. 

We choose to parametrize the mass as an element of the spaces of non-negative numbers $m \in \mathbb{R}_{\geq 0}$. 

The center of mass do not have any constraints on its location, so we choose to parametrize it as an element of the 3D space $c \in \mathbb{R}^3$. 

For parametrize the 3D inertia matrix ensuring the properties described in subsection \ref{subsec:principalAxis} we choose the \emph{second moments of mass} $L \in \mathbb{R}^3_{\geq 0}$ to be components of our parametrization. In the following, we will show how this choice ensures the \emph{full physical consistency} of the resulting inertial parameters. 



The inertial parameters of a rigid body can then be parametrized by an element $\theta \in \mathfrak{P} = \mathbb{R}_{\ge 0} \times \mathbb{R}^3 \times  SO(3) \times \mathbb{R}_{\ge 0}^3$. In particular the components of $\theta$ are:
\begin{itemize}
    \item $m \in \mathbb{R}_{\ge 0}$ the mass of the body 
    \item $c \in \mathbb{R}^3$ the center of mass of the body 
    \item $Q \in SO(3)$ the rotation matrix between the body frame and the frame of principal axis at the center of mass
    \item $L \in \mathbb{R}_{\ge 0}^3$ the second central moment of mass along the principal axes 
\end{itemize}

In other terms, there is a function $\pi_p(\theta) : \mathfrak{P} \mapsto \mathbb{R}^{10}$ that maps this new parametrization to the corresponding inertial parameters:
\begin{IEEEeqnarray}{rCCCl}
\label{eq:pip}
\pi_p(\theta) 
&=&
    \begin{bsmallmatrix}
    m(\theta) \\
    mc(\theta) \\
    \operatorname{vech}\left(I_B(\theta)\right) \\
    \end{bsmallmatrix}  
    &=& 
    \begin{bsmallmatrix}
    m \\
    mc \\
    \operatorname{vech}\left( Q \operatorname{diag}{(PL)}  Q^\top - m S(c) S(c) \right) \\
    \end{bsmallmatrix} \nonumber
\end{IEEEeqnarray}
Where $P = \left[ \begin{smallmatrix}
0 & 1 & 1 \\
1 & 0 & 1 \\
1 & 1 & 0 
\end{smallmatrix} \right]$ is a matrix that maps $L$ to $J$.

\begin{theorem}
\label{thm:mainTheorem}
For each $\theta \in \mathfrak{P}$, there exists a density function $\rho(\cdot) : R^3 \mapsto R_{\geq 0}$ such that $\pi_d(\rho(\cdot)) = \pi_p(\theta)$, i.e. every  $\theta \in \mathfrak{P}$ generates \emph{fully physical consistent} inertial parameters.
\end{theorem}
\begin{proof}
The proof is given in Appendix \ref{proof:mainTheorem}.
\end{proof}

Using this parametrization, it is possible to recast the identification optimization problem \eqref{eq:optimizationProblemLinear} as:
\begin{IEEEeqnarray}{rCl}
\label{eq:optimizationProblemNonLinear}
 \hat{\pi} &=& \pi(\hat{\theta}) \\
\hat{\theta} &=&  \argmin_{\theta \in \mathfrak{P}} \sum_{i = 1}^N \left\| Y(\rma^g_i, \rmv_i) \pi(\theta) - \rmf_i \right\|^2
\end{IEEEeqnarray}

The main advantage of \eqref{eq:optimizationProblemNonLinear} with respect to \eqref{eq:optimizationProblemLinear} is that thanks to Theorem \ref{thm:mainTheorem} the identified inertial parameters $\hat{\pi}$ are ensured to be \emph{fully physically consistent}. However, the optimization variable $\theta$ does not live anymore in a Euclidean space, because $\mathfrak{P}$ includes $SO(3)$, so to solve this optimization problem we exploit specific techniques related to the \emph{optimization on manifolds}.

\subsection{Optimization on manifolds}

In a nutshell, an n-dimensional manifold $\mathcal{M}$ is a space that locally looks like $\mathbb{R}^n$, \emph{i.e.}, for any point $x$ of $\mathcal{M}$ there exists a smooth map $\varphi_x$ between an open set of $\mathbb{R}^n$ and a neighborhood of $x$, with $\varphi_x(0) = x$.

For this paper, we focus on $SO(3)$, a 3-dimensional manifold. As such, it can be parametrized \emph{locally} by $3$ variables, for example, a choice of Euler angles, but any such parametrization necessarily exhibits singularities when taken as a global map (e.g. gimbal lock for Euler angles), which can be detrimental to our optimization process.

For this reason, when addressing $SO(3)$ with classical optimization algorithms, it is often preferred to use one of the two following parametrizations:
\begin{itemize}
    \item unit quaternion, \emph{i.e.} an element $q$ of $\mathbb{R}^4$ with the additional constraint $\left\|q\right\| = 1$,
    \item rotation matrix, \emph{i.e.} an element $R$ of $\mathbb{R}^{3 \times 3}$ with the additional constraints $R^T R = I$ and $\det{R} \geq 0$. 
\end{itemize}

The alternative is to use optimization software working natively with manifolds~\cite{brossette2015humanoid}\cite{absil:book:2008} and solve
\begin{align}
\label{eq:finalProblem}
    \argmin_{\theta \in \mathbb{R}\times\mathbb{R}^3\times SO(3) \times \mathbb{R}^3} &\ \sum_{i = 1}^N \left\| Y(\rma^g_i, \rmv_i) \pi(\theta) - \rmf_i \right\|^2 \\
    \mbox{subj. to} &\ m \geq 0,\ L_x \geq 0,\ L_y \geq 0,\ L_z \geq 0
\end{align}

This alternative has an immediate advantage: we can write directly the problem \eqref{eq:optimizationProblemNonLinear} without the need to add any parametrization-related constraints. Because there are fewer variables and fewer constraints, it is also faster to solve. To check this, we compared the resolution of~\eqref{eq:optimizationProblemNonLinear} formulated with each of the three parametrizations (native $SO(3)$, unit quaternion, rotation matrix). We solved the three formulations with the solver presented in~\cite{brossette2015humanoid}, and the two last with an off-the-shelf solver (CFSQP~\cite{cfsqp:manual}), using the dataset presented in section~\ref{sec:experiments}. 
The formulation with native $SO(3)$ was consistently solved faster. We observed timings around $0.5$s for it, and over $1$s for non-manifold formulations with CFSQP. The mean time for an iteration was also the lowest with the native formulation (at least $30\%$ when compared to all other possibilities).

Working directly with manifolds has also an advantage that we do not leverage here, but could be useful for future work: at each iteration, the variables of the problem represent a fully physical consistent set of inertial parameters.
This is not the case with the other formulations we discussed, as the (additional) constraints are guaranteed to be satisfied only at the end of the optimization process. 
Having physically meaningful intermediate values can be useful to evaluate additional functions that presuppose it (additional constraints, external monitoring $\ldots$). 
It can also be leveraged for real-time applications where only a short time is allocated repeatedly to the inertial identification, so that when the optimization process is stopped after a few iterations, the output is physically valid.
With non-manifold formulations, at any given iteration, the parametrization-related constraints can be violated, thus, the variables might not lie in the manifold. It is then needed to project them on it. Denoting $\pi$ the projection (for example $\pi = \frac{q}{\left\|q\right\|}$ in the unit quaternion formulation), to evaluate a function $f$ on a manifold, we need to compute $f \circ \pi$. If further the gradient is needed, that projection must also be accounted for (\cite{bouyarmane2012humanoids} explains this issue in great details for free-floating robots).

In this study, we use the same solver and approach as presented in \cite{brossette2015humanoid} which was inspired from \cite{absil:book:2008}.
The driving idea of the optimization on manifold is to change the parametrization at each iteration. The problem at iteration k becomes:
\begin{IEEEeqnarray}{rClrCl}
  \min_{z_k \in \mathbb{R}^n}\ & f\circ\varphi_{x_k}(z) \quad
  \text{s.t.} \quad & c\circ\varphi_{x_k}(z) = 0 .
\end{IEEEeqnarray}
Then $x_{k+1} = \varphi_{x_k}(z_k)$ is guaranteed to belong to $\mathcal{M}$. The next iteration uses the same formulation around $x_{k+1}$.


The smooth maps $\varphi_x$ are built-in and are used automatically by the solver while the user only has to implement the functions of the optimization problem without the burden of worrying about the parametrization.

\section{EXPERIMENTS}
\label{sec:experiments}
The iCub is a full-body humanoid with 53 degrees of freedom \cite{metta2010icub}.  For validating the presented approach, we used the six-axis force/torque (F/T) sensor embedded in iCub's right arm to collect experimental F/T measurements. We locked the elbow, wrist and hands joints of the arm, simulating the presence of a rigid body directly attached to the F/T sensor, a scenario similar to the one in which an unknown payload needs to be identified \cite{kubus2008line}. 
\begin{figure}[htb]
\begin{overpic}[width=0.48\textwidth,natwidth=1235,natheight=742]{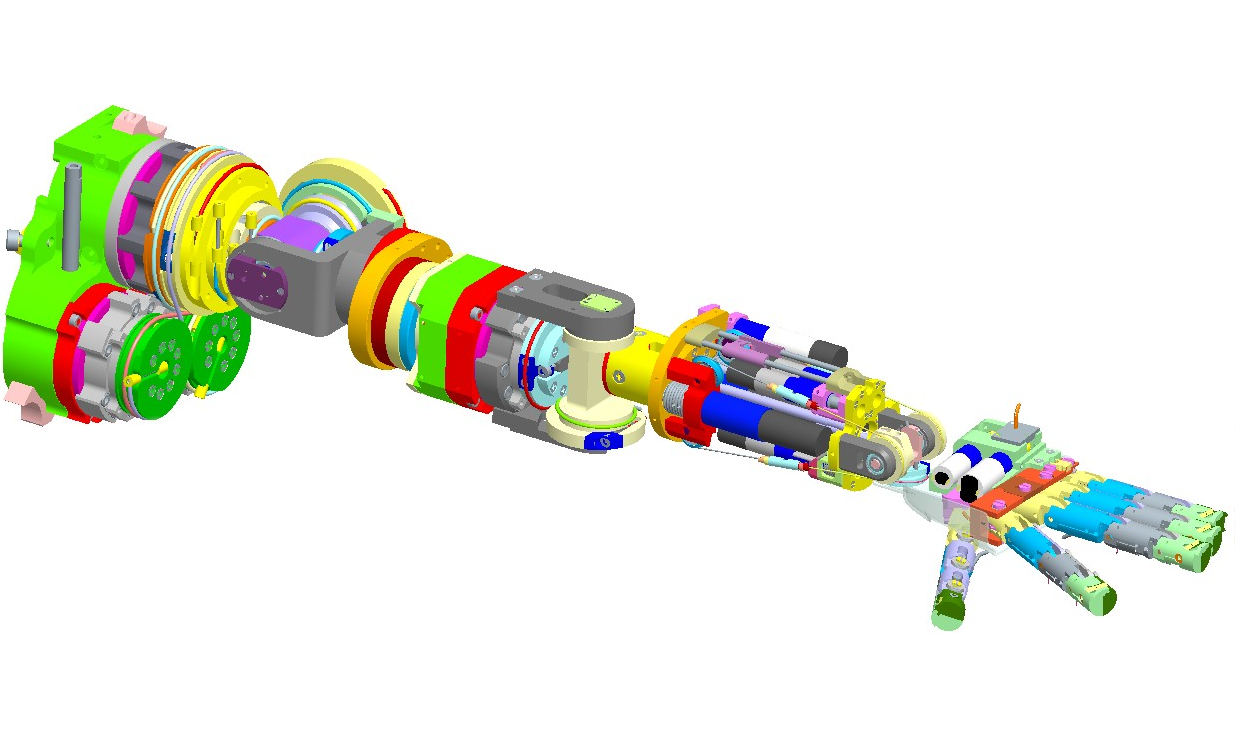}
\put(5,10){FT sensor}
\put(13,13){\vector(1,1){18}}
\put(38,50){Upper arm}
\put(43,49){\vector(0,-1){12}}
\put(65,45){Forearm}
\put(70,44){\vector(-1,-2){7}}
\end{overpic}
\caption{CAD drawing of the iCub arm used in the experiments. The used six-axis F/T sensor is visible in the middle of the upper arm link.}
\label{fig:cadArm}
\end{figure}
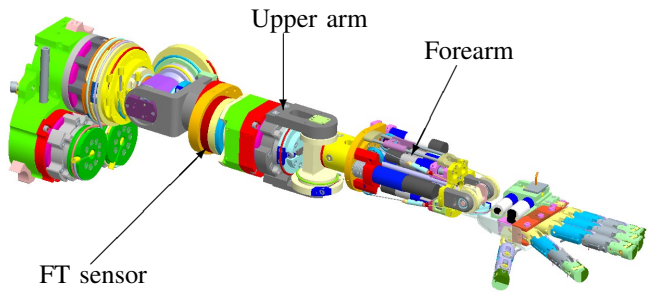

We generated five 60 seconds joint positions paths in which the three shoulder joints were reaching random joint position using point to point minimum-jerk like trajectories. The point to point trajectory completion times were $10$s, $5$s, $2$s, $1$s and $0.5$s for the different paths. 
We played these joint paths on the robot and we sampled at $100$Hz the F/T sensors and joint encoders output. We filtered the joint positions and obtained joint velocities and accelerations using a Savitzky-Golay filtering of order 2 and with a windows size of $499$, $41$, $21$, $9$, $7$ samples. We used joint positions, velocities and accelerations with the kinematic model of the robot to compute $\rma^g$ and $\rmv$ of the F/T sensor for each time sample.
We removed the unknown offset from the F/T measurements using the offset removal technique described in \cite{traversaro2015situ}.
We then solved the inertial identification problem using the classical linear algorithm \eqref{eq:optimizationProblemLinear} and the one using the proposed \emph{fully physical consistent} parametrization \eqref{eq:finalProblem}.
We report the identified inertial parameters in Table \ref{table:results}. 
It is interesting to highlight that for slow datasets (trajectory time of $10$s or $5$s) the unconstrained optimization problem \eqref{eq:optimizationProblemLinear} results in inertial parameters that are not fully physical consistency. 
In particular, this is due to the low values of angular velocities and acceleration, that do not properly excites the inertial parameters, which are then \emph{numerically not identifiable}. 
The proposed optimization problem clearly cannot identify these parameters anyway, as the identified parameters are an order of magnitude larger than the ones estimated for faster datasets, nevertheless, it always estimates inertial parameters that are fully physical consistent. 
For faster datasets (trajectory time of $1$s or $0.5$s) the results of the two optimization problems are the same because the high values of angular velocities and accelerations permit to identify all the parameters perfectly. 
While this is possible to identify all the inertial parameters of a single rigid body, this is not the case when identifying the inertial parameters of a complex structure such as a humanoid robot, for which both structural \cite{ayusawa2014identifiability} and numerical \cite{pham1991essential} not identifiable parameters exists. 
In this later application, the enforcement of full physical consistency  will always be necessary to get meaningful results.

 
 
\begin{table*}[ht]
\caption{Inertial parameters identified with the different datasets and the different optimization problems.} 
\begin{center}
\begin{tabular}{ |c|c|c|c|c|c|c|c|c|c|c|c|}
\hline
 \rowcolor[gray]{.9} \textbf{Trajectory Time} & \textbf{Optimization Manifold} & $\mathbf{m}$ & $\mathbf{mc_x}$ & $\mathbf{mc_y}$ & $\mathbf{mc_z}$  & $\mathbf{I_{xx}}$ & $\mathbf{I_{xy}}$  & $\mathbf{I_{xz}}$  & $\mathbf{I_{yy}}$  & $\mathbf{I_{yz}}$ & $\mathbf{I_{zz}}$ \\
\hline
    & $\mathbb{R}^{10}$ & 1.836 & 0.062 & 0.001 & 0.208 & \cellcolor[HTML]{FFCCCC} 0.580 & \cellcolor[HTML]{FFCCCC} 0.593 & \cellcolor[HTML]{FFCCCC} -0.541 &  \cellcolor[HTML]{FFCCCC}1.022 & \cellcolor[HTML]{FFCCCC} 0.190 & \cellcolor[HTML]{FFCCCC} -0.129 \\\cline{2-12} 
    \rowcolor[gray]{.9} \multirow{-2}{4em}{\cellcolor{white}10s}  & $\ipspace$ & 1.836 & 0.062 & 0.001 & 0.208 & 0.215 & 0.012 & -0.064 & 0.227 & 0.038 & 0.028 \\ 
\hline
& $\mathbb{R}^{10}$ & 1.842 & 0.061 & 0.000 & 0.206 & \cellcolor[HTML]{FFCCCC} 0.128 & \cellcolor[HTML]{FFCCCC} -0.018 & \cellcolor[HTML]{FFCCCC} -0.125 & \cellcolor[HTML]{FFCCCC} 0.125 & \cellcolor[HTML]{FFCCCC} 0.026 & \cellcolor[HTML]{FFCCCC} -0.001 \\\cline{2-12} 
\rowcolor[gray]{.9} \multirow{-2}{4em}{\cellcolor{white} 5s}  & $\ipspace$ & 1.842 & 0.060 & 0.000 & 0.206 & 0.166 & 0.001 & -0.089 & 0.216 & 0.001 & 0.050 \\ 
\hline
 & $\mathbb{R}^{10}$ & 1.852 & 0.060 & 0.001 & 0.206 & 0.065 & 0.001 & -0.035 & 0.066 & 0.006 & 0.007 \\\cline{2-12} 
\rowcolor[gray]{.9} \multirow{-2}{4em}{\cellcolor{white} 2s} & $\ipspace$ & 1.852 & 0.060 & 0.001 & 0.206 & 0.067 & 0.001 & -0.030 & 0.086 & 0.003 & 0.014 \\ 
\hline
& $\mathbb{R}^{10}$ & 1.820 & 0.060 & 0.002 & 0.205 & 0.032 & 0.0014 & -0.017 & 0.036 & 0.002 & 0.008 \\\cline{2-12} 
\rowcolor[gray]{.9} \multirow{-2}{4em}{\cellcolor{white}  1s} & $\ipspace$ & 1.820 & 0.060 & 0.002 & 0.205 & 0.034 & 0.001 & -0.015 & 0.042 & 0.001 & 0.009 \\ 
\hline
& $\mathbb{R}^{10}$ & 1.843 & 0.060 & 0.005 & 0.204 & 0.033 & 0.003 & -0.014 & 0.035 & 0.000 & 0.008 \\\cline{2-12} 
\rowcolor[gray]{.9} \multirow{-2}{4em}{\cellcolor{white} 0.5s} & $\ipspace$ & 1.844 & 0.059 & 0.004 & 0.204 & 0.037 & 0.001 & -0.013 & 0.039 & 0.000 & 0.008 \\ 
\hline
\end{tabular}
\end{center}
\label{table:results}

Inertial parameters identified on $\mathbb{R}^{10}$ optimization manifold that are not fully physical consistent are highlighted.

Masses are expressed in $kg$, first moment of masses in $kg.m$, inertia matrix elements in $kg.m^2$.

\end{table*}

\section{CONCLUSIONS}
The condition for the \emph{full physical consistency} of the rigid body inertial parameters was introduced, to extend the existing concept of the \emph{physical consistency} of the inertial parameters. A nonlinear parametrization of the inertial parameters that ensures \emph{full physical consistency} was also introduced and a nonlinear optimization on manifolds technique was adapted to solve the resulting nonlinear identification problem. The results have been validated with experiments on the identification of the inertial parameters of the right arm of the iCub humanoid robot.

For the sake of simplicity and for space constraints, in this work, only the case of the identification of a single rigid body was discussed. However, the \emph{full physical consistency} concept and the optimization on manifolds are general concepts, that we plan to integrate with existing techniques for whole body inertial parameters identification in humanoids \cite{ayusawa2014identifiability,jovic2015identification} and for adaptive control of underactuated robots \cite{pucci2015collocated}.



\section*{APPENDIX}


\subsection{Proof of Theorem 1}
\begin{proof}
\label{proof:mainTheorem}
We prove the statement in a constructive way: given an arbitrary element $\theta = (m,c,Q,L) \in \ipspace$ we build a density function $\rho(\cdot): \mathbb{R}^3 \mapsto \mathbb{R}_{\geq 0}$ such that $\pi_p(\theta) = \pi_d(\rho(\cdot))$. 
For example we can think of a cuboid of uniform unit density, with the center of the cuboid coincident with the center of mass of the inertial parameters (given by $c$), with the orientation of its symmetry axis aligned with the $C$ \emph{principal axes} frame defined by the $Q$ rotation matrix and the cuboid sides of lengths $2 d_x$ , $2 d_y$ and $2 d_z$, with:
\begin{equation}
\label{eq:cuboidSize}
    d = 
    \begin{bmatrix}
    d_x &
    d_y &
    d_z 
    \end{bmatrix}^{\top}
    = 
    \begin{bmatrix}
    \sqrt{3 \frac{L_x}{m}} &
    \sqrt{3 \frac{L_y}{m}} &
    \sqrt{3 \frac{L_z}{m}} 
    \end{bmatrix}^{\top}
\end{equation}

Its density function in the $C$ frame is given as:
\begin{equation*}
    \tilde{\rho}(\tilde{r}) = 
  \begin{cases} 
      \hfill 1 \hfill & \text{ if $-d \geq \tilde{r} \geq d$} \\
      \hfill 0 \hfill & \text{otherwise} \\
  \end{cases}
\end{equation*}
while the density function in the $B$ frame is given by:
\begin{equation}
  \label{eq:equivalentCuboidDensity}
    {\rho}({r}) = 
  \begin{cases} 
      \hfill 1 \hfill & \text{ if $-Qd+c \geq {r} \geq Qd+c$} \\
      \hfill 0 \hfill & \text{otherwise} \\
  \end{cases}.
\end{equation}

The density defined in \eqref{eq:equivalentCuboidDensity} and \eqref{eq:cuboidSize} can be seen as a function $\gamma(\cdot) : \ipspace \mapsto (\mathbb{R}^3 \mapsto \mathbb{R}_{\geq 0})$. The theorem is then demonstrated by using \eqref{eq:pid} and \eqref{eq:pip} to verify that:
$$
\pi_d(\gamma(\theta)) = \pi_p(\theta)
$$
is true $\forall~\theta \in \ipspace$.
\end{proof}


\bibliographystyle{IEEEtran}  
\bibliography{root}  

\addtolength{\textheight}{-12cm} 

\end{document}